\documentclass[journal]{IEEEtran}
\IEEEoverridecommandlockouts
\usepackage{cite}

\usepackage{amsmath,amssymb,amsfonts,mathtools, bm}
\usepackage{amsthm}
\usepackage{algorithm}
\usepackage{algorithmic}
\usepackage{graphicx}
\usepackage{textcomp}
\usepackage{xcolor,float}
\usepackage{tabularx}
\usepackage{textcomp}
\usepackage{diagbox}
\usepackage{svg}
\usepackage{booktabs}
\usepackage{subfigure}
\usepackage{hyperref}
\usepackage{subcaption}
\newtheorem{theorem}{Theorem}
\newtheorem{lemma}{Lemma}
\newtheorem{corollary}{Corollary}
\newtheorem{problem}{Problem}
  
\usepackage{graphicx}

\allowdisplaybreaks
\renewcommand{\baselinestretch}{1}

\begin{document}

\title{Reliable Projection Based Unsupervised Learning for Semi-Definite QCQP with Application of Beamforming Optimization}
\author{
Xiucheng Wang,~\IEEEmembership{Student Member,~IEEE,}
Qi Qiu,
Nan Cheng,~\IEEEmembership{Senior Member,~IEEE,}

\thanks{ }
\thanks{
\par Xiucheng Wang, Nan Cheng are with the State Key Laboratory of ISN and School of Telecommunications Engineering, Xidian University, Xi’an 710071, China (e-mail: xcwang\_1@stu.xidian.edu.cn; dr.nan.cheng@ieee.org). \textit{Nan Cheng is the corresponding author}.
\par Qi Qiu is with the School of Telecommunications Engineering, Xidian University, Xi'an, 710071, China (e-mail: qiuqi@stu.xidian.edu.cn).
}

}
    
    \maketitle

\IEEEdisplaynontitleabstractindextext

\IEEEpeerreviewmaketitle

\begin{abstract}
In this paper, we investigate a special class of quadratic-constrained quadratic programming (QCQP) with semi-definite constraints. Traditionally, since such a problem is non-convex and N-hard, the neural network (NN) is regarded as a promising method to obtain a high-performing solution. However, due to the inherent prediction error, it is challenging to ensure all solution output by the NN is feasible. Although some existing methods propose some naive methods, they only focus on reducing the constraint violation probability, where not all solutions are feasibly guaranteed. To deal with the above challenge, in this paper a computing efficient and reliable projection is proposed, where all solution output by the NN are ensured to be feasible. Moreover, unsupervised learning is used, so the NN can be trained effectively and efficiently without labels. Theoretically, the solution of the NN after projection is proven to be feasible, and we also prove the projection method can enhance the convergence performance and speed of the NN. To evaluate our proposed method, the quality of service (QoS)-contained beamforming scenario is studied, where the simulation results show the proposed method can achieve high-performance which is competitive with the lower bound.
\end{abstract}


\section{Introduction}

With the development of 6G networks towards service-driven paradigms that prioritize user preferences, the optimization process in these networks faces an inevitable challenge of solving constraint optimization (CO) problems based on personalized user preferences and service performance constraints \cite{6g}. While traditional iterative methods have demonstrated success in solving problems with precise convex constraints or providing approximate solutions for simple non-convex constraints, their iterative process often involves computationally intensive operations such as computing Hessian matrices or inverses \cite{li2014special,he2022qcqp}. As a result, they fail to address real-time CO problems that require timely solutions, such as beamforming design issues with user rate preference constraints. This limitation undermines their applicability in time-sensitive applications. To overcome these challenges, researchers have turned to computationally efficient deep neural networks (NN), employing the learning to optimize (L2O) approach for solving CO problems. However, the utilization of L2O is not without flaws \cite{wang2023scalable, wang2022joint}. The black-box nature of NNs, their data dependencies, deterministic parameters, and scaling law pose ongoing challenges in terms of efficiency, constraint handling, performance, and training costs that researchers cannot overlook

While the universal approximation theorem suggests that NNs have the capability to model any mapping, thus theoretically implying the feasibility of designing an NN that can project input optimization problems into a feasible solution space, this assertion assumes the use of infinitely wide networks, which is impractical in real-world wireless network systems. When employing finite-width NNs, mapping errors are inevitable, leading to outputs that lie outside the feasible region \cite{goodfellow2016deep}. Consequently, existing methods that utilize L2O for solving CO problems often struggle to achieve satisfactory performance while maintaining low computational complexity. Even conventional approaches, such as constructing penalty functions associated with constraints to solve CO problems through unconstrained Lagrangian dual functions, fail to ensure computational efficiency \cite{shen2020graph,wang2023scalable,wang2022joint}. This is due to the increased complexity of the Lagrangian dual functions compared to the original CO problem's objective functions, which in turn makes it more challenging to fit NNs. Moreover, these methods do not guarantee the feasibility of constraints as non-convex Lagrangian dual functions may lead NN to fall into local minima outside the feasible region. Existing feasible L2O methods, such as those based on sampling, require repetitive sampling within a fixed sampling space until a feasible point is found for the final output. However, this approach is computationally inefficient and limited to fixed constraint equations, unable to handle variable constraints. Such limitations contradict the diverse service requirements and personalized development goals of 6G networks, where user demands are dynamic and not always identical.

Hence, it seems to increasingly necessary to enhance the model's representational capacity to ensure the reliability of mappings from optimization problems to feasible solution spaces. As a result, NN architectures employed for wireless network beamforming optimization have evolved into more complex structures, transitioning from simple Multilayer Perceptrons (MLPs) \cite{cao2023pcnet} to Convolutional NNs (CNNs) \cite{wang2023cnn}, Long Short-Term Memory (LSTM) networks, and further to Graph NNs (GNNs) and Transformers \cite{shen2020graph, li2022radionet}. These advanced architectures possess unprecedented representational capabilities, enabling them to effectively tackle CO problems in radio frequency resource management. However, indiscriminate model size expansion offers minimal benefits for solving CO problems while introducing significant drawbacks. Larger models typically lead to longer inference delays, which conflicts with the trend of reducing the air interface transmission time interval (TTI) in the transition from 4G to 5G and toward future 6G networks. The reduction of TTI aims to enhance the agility and real-time capabilities of wireless radio frequency management \cite{gu2021knowledge}. Consequently, the adoption of oversized models results in excessive inference delays, making their practical application in wireless network radio frequency resource management challenging. Furthermore, even when expanding the model parameters to the scale of hundreds of billions, there is no guarantee of producing constraint-compliant outputs. This is analogous to how even GPT-4 may occasionally generate sentences that do not adhere to grammatical and spelling rules.

Are we then left with no recourse but to accept the possibility of L2O methods yielding solutions that violate constraints when solving CO problems? Upon examining the existing research on L2O for CO in wireless networks, it is surprising to discover that optimization problems with one-dimensional linear summation constraints, such as transmission rate optimization under power constraints or task size-constrained edge computing task offloading, are rarely treated as CO problems. Researchers often assume that the constraints can be satisfied, and experimental results also demonstrate no violations of constraints. This phenomenon can be attributed to the utilization of modern NN architectures, such as TensorFlow or PyTorch, which incorporate softmax functions. These functions enable the projection of NN outputs into a solution space that sums to one, easily meeting any first-order linear summation constraint. Could we, therefore, devise an ingenious projection function that maps the solutions of beamforming optimization problems into a feasible region, satisfying all user rate preferences and compelling any NN output to comply with the constraints?

However, there are instances where finding a feasible point within the region of the CO problem proves to be as challenging as solving the original problem itself. Therefore, the proposed projection function must exhibit computational efficiency and avoid becoming a bottleneck for L2O efficiency. Additionally, the projection function should possess the following essential properties: (1) \textbf{Initial Point Dependency and Determinism} : This necessitates that when a point lies outside the feasible region, its projected point within the feasible region should be related to the original point's coordinates, rather than projecting to a random feasible point. Failure to establish this relationship would hinder the NN's ability to learn the connection between its outputs and the projected points, thereby impeding the optimization of the CO problem. Ideally, the projection function should be a bijection, enabling the NN to comprehend how its outputs, even when outside the feasible region, can influence the objective function of the CO problem; (2) \textbf{Performance Preservation}: The projection function should not compromise the solution performance of the CO problem. This implies that the set of solutions after projection should include the subset of the feasible region containing the optimal points of the original CO problem. If the set of projected solutions does not encompass the optimal points, the mapping would diminish the upper bound of system performance; (3) \textbf{Performance Enhancement}: Ideally, the mapping function should also enhance the probability of randomly sampling the global or local optimum of the original CO problem within the mapping interval.

Although a projection function incorporating the aforementioned properties can ensure constraint compliance in solving CO problems, the data-driven L2O approach still necessitates extensive training data for optimal performance. However, generating a substantial volume of high-quality solutions as labels for supervised training is exceedingly time-consuming in the case of CO problems, particularly the NP-hard beamforming design under user rate preference constraints. Acquiring a sufficient number of training labels within an acceptable timeframe to train NNs is impractical. Furthermore, despite having an ample supply of labels, NNs may not surpass the performance of iterative algorithms, as the essence of NNs, according to the Universal Approximation Theorem, lies in fitting high-dimensional mappings. Although reinforcement learning-based methods can alleviate the reliance on optimal labels, their protracted convergence times and unstable performance render them unsuitable for solving complex CO problems. To address these challenges, this paper proposes a differentiable, initial point deterministic, performance-guaranteed, and L-Lipschitz continuous projection function. This function enables the reliable, rapid, and high-performance solution of beamforming optimization problems constrained by user transmission rate demands. By explicitly deriving gradients of NN parameters with respect to projection points and optimization objectives, a label-free, unsupervised training approach is implemented for swift NN solution. The primary contributions of this paper are as follows.
\begin{enumerate}
    \item To our best knowledge, in this paper, for the first time, we propose a method that can make all solutions satisfy the constraint condition when using L2O to solve a beamforming constrained optimization problem with user rate preference. Instead of focusing only on the probability that the constraint is satisfied as in the previous work, in other words, we propose a method that can make the constraint satisfy the probability micro 100\%.
    \item Through an analysis of the Karush-Kuhn-Tucker (KKT) conditions for the NP-hard beamforming constraint optimization problem constrained by user transmission rate requirements, we propose a performance-guaranteed projection function. This function projects NN outputs into the feasible region, ensuring that the projected outputs satisfy specific KKT conditions, thereby guaranteeing their post-projection performance.
    \item We theoretically and qualitatively analyze that utilizing our proposed projection function can accelerate training convergence while reducing the representational demands on the NN. This allows for the use of simple MLP with only one hidden layer to solve the beamforming constraint optimization problem constrained by user transmission rate requirements.
    \item Leveraging the differentiability of our proposed projection function, we explicitly derive the gradients of NN parameters with respect to the optimization objectives. This enables training the NN in a label-free, unsupervised manner, significantly reducing training costs.
    \item Simulation results demonstrate that the method proposed in this paper achieves performance close to the theoretical lower bound within a short timeframe.
\end{enumerate}

\section{System Model and Optimality Analysis}
In this paper, we consider a single-cell, multiuser multiple-input single-output (MISO) downlink system. The system comprises a Base Station (BS) equipped with $ N $ antennas, serving $ M $ single-antenna users within its cell. This setup is designed to cater to the escalating demand for high-data-rate applications, driven by the widespread adoption of wireless devices and the proliferation of media-rich content. The wireless channel, characterized by its inherent unpredictability due to multipath propagation, is represented by the channel vector $\bm{h}_i \in \mathbb{C}^N$ for the $ i $-th user. This vector captures the complex gains between the BS antennas and the receiver of the user. To facilitate transmission, a beamforming vector $ \bm{w} $ is employed, which requires careful design to manage transmitted power while ensuring quality of service (QoS) for each user, as measured by the signal-to-noise ratio (SNR).

The design of the beamforming vector is formulated as an optimization problem with the objective of minimizing the transmit power $ \| \bm{w} \|^2 $, subject to individual SNR constraints for each user. These constraints ensure that the received power $ \|\bm{h}_i^H \bm{w}\|^2 $ normalized by the noise variance $ \sigma_i^2 $ at the receiver exceeds a predetermined threshold $ \gamma_i $, guaranteeing a minimum SNR for reliable communication. This optimization problem reflects the multicast nature of the communication, where common information is transmitted to all users simultaneously. Hence, the problem can be formulated as follows.
\begin{problem}\label{p1}
\begin{align}
& \min_{\bm{w} \in \mathbb{C}^N}
& & \| \bm{w} \|^2 \\
& \text{s. t.}
& & \frac{\|\bm{h}_i^H \bm{w}\|^2}{\sigma_i^2} \geq \gamma_i, \; \forall i\in \{1, \cdots, M\},
\end{align}
\end{problem}
    
\begin{theorem}\label{theorem-1}
    The Problem \ref{p1} is a Non-deterministic Polynomial (NP)-hard problem.
\end{theorem}

\begin{proof}
    According to \cite{sidiropoulos2006transmit} problems of the following form are known to be NP-hard
\begin{align*}
&  \min_{\bm{x} \in \mathbb{C}^N}
& & \bm{x}^H \bm{A}_0 \bm{x}\\
& s. t.
& & \bm{x}^H \bm{A}_{i} \bm{x}\leq c_i, \; \forall i\in \{1, \cdots, M\},
\end{align*}
The Problem \ref{p1} reformulated into the above form by setting $\bm{A}_0=\bm{I}$, $\bm{A}_i=-\bm{h}_i\bm{h}^{H}_i$, and $c_i=-\sigma_i\gamma_i$, for $\forall i\in \{1, \cdots, M\}$, where $\bm{I}$ denotes the identity matrix. Consequently, Problem \ref{p1} is classified as an NP-hard problem.
\end{proof}

The NP-hard nature of the optimization problem, as established previously, poses a significant challenge in terms of effective and efficient resolution using traditional convex optimization methods. This challenge becomes even more pronounced in the context of user on-demand services within 6G wireless networks, where low latency and computational efficiency are crucial. The pursuit of ubiquitous responsiveness in the 6G ecosystem, aiming to address user demands promptly and accurately, is hindered by the latency introduced by conventional optimization techniques. Therefore, it is imperative to conduct a detailed analysis of the problem's characteristics to identify an effective method for its solution. Fortunately, through further analysis, we can derive certain properties that can serve as the foundation for designing a powerful approach to address the problem.

\begin{theorem}\label{theorem-2}
    The optimal solution to Problem \ref{p1} must be on the boundary of the feasible region
\end{theorem}
\begin{proof}
The Lagrangian dual function $ \mathcal{L} $ for Problem \ref{p1} can formulated as follows.
\begin{align}
    \mathcal{L}(\bm{w}, \bm{\lambda}) = \|\bm{w}\|^2 - \sum_{i=1}^M \lambda_i \left( \gamma_i - \frac{\|\bm{h}_i^H \bm{w}\|^2}{\sigma_i^2} \right),\label{kkt}
\end{align}
where $ \bm{w} $ signifies the beamforming vector under optimization, $ \bm{\lambda} = [\lambda_1, \lambda_2, \cdots, \lambda_M]^T $ encapsulates the non-negative Lagrange multipliers corresponding to the constraints, $ \gamma_i $ denotes the SNR requisites, $ \sigma_i^2 $ is the noise variance for the $ i $-th user, and $ \bm{h}_i $ represents the channel vector from the BS to the $ i $-th user. As elucidated by \cite{boyd2004convex}, the optimal solution is characterized by several critical properties, articulated as follows.
\begin{align*}
   &\nabla_{\bm{w}} \mathcal{L}(\bm{w}, \bm{\lambda}) = 2\bm{w} - \sum_{i=1}^M \lambda_i \left( \frac{2\bm{h}_i \bm{h}_i^H}{\sigma_i^2} \right) \bm{w} = \bm{0},\tag{\ref{kkt}a}\label{kkt-1}\\
   &\lambda_i \geq 0, \quad \forall i \in \{1, \cdots, M\},\tag{\ref{kkt}b}\label{kkt-2}\\
   &\lambda_i \left( \gamma_i - \frac{\|\bm{h}_i^H \bm{w}\|^2}{\sigma_i^2} \right) = 0, \quad \forall i \in \{1, \cdots, M\},\tag{\ref{kkt}c}\label{kkt-3}\\
   &\frac{\|\bm{h}_i^H \bm{w}\|^2}{\sigma_i^2} \geq \gamma_i, \quad \forall i\in \{1, \cdots, M\}.\tag{\ref{kkt}d}\label{kkt-4}
\end{align*}
The complementary slackness condition, as is encapsulated in equation \eqref{kkt-3}, intimates that for the optimal solution, each Lagrange multiplier $ \lambda_i $ is either zero or the constraint $\gamma_i -\frac{\|\bm {h}_i^H\bm{w}\|^2}{\sigma_i^2} = 0$ is active. Nonetheless, were all $ \lambda_i $ to be zero, a reformulation of the stationarity condition would yield as follows.
\begin{align}
    \nabla_{\bm{w}} \mathcal{L}(\bm{w}, \bm{\lambda}) = 2\bm{w},\tag{\ref{kkt}e}\label{kkt-5}
\end{align}
As established in \cite{boyd2004convex}, the gradient of the Lagrangian dual function with respect to the optimal solution $ \bm{w}^* $ ought to be zero. Nonetheless, an examination of the function \eqref{kkt-5} reveals that its gradient with respect to $ \bm{w} $ zeroifies exclusively when $ \bm{w} = \bm{0} $. Such a condition, $ \bm{w} = \bm{0} $, cannot be reconciled with the primal feasibility condition as expressed in \eqref{kkt-4}. Consequently, it is necessitated that the optimal solution involves at least one Lagrange multiplier $ \lambda_i $ differing from zero, indicating the presence of at least one index $ \forall i\in \{1, \cdots, M\} $ for which the constraint $ \gamma_i - \frac{\|\bm{h}_i^H \bm{w}\|^2}{\sigma_i^2} = 0 $ holds true. It therefore follows that the optimal solution to Problem \ref{p1} invariably lies on the boundary of the feasible region.
\end{proof}

\begin{figure*}
    \centering
    \includegraphics[width=0.9\linewidth]{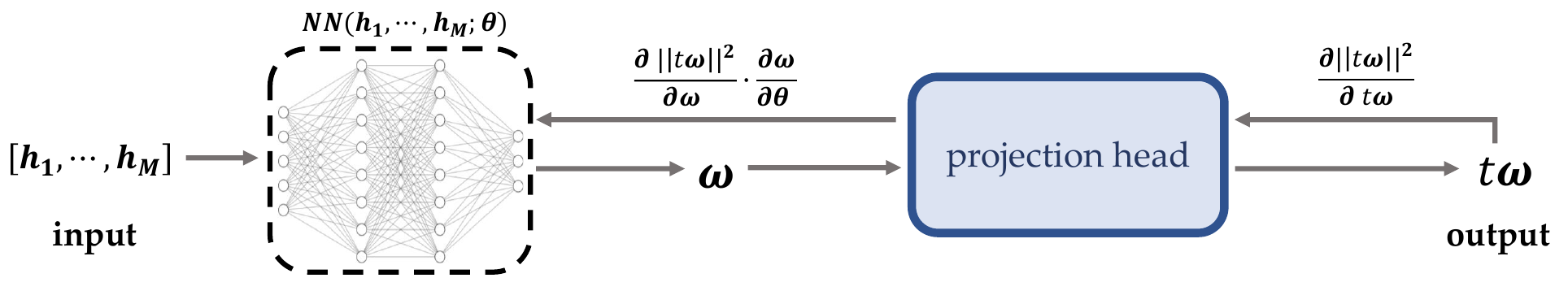}
    \caption{Inference and training procedure.}
    \label{fig:system_inference}
\end{figure*}

\begin{figure*}
    \centering
    \includegraphics[width=0.9\linewidth]{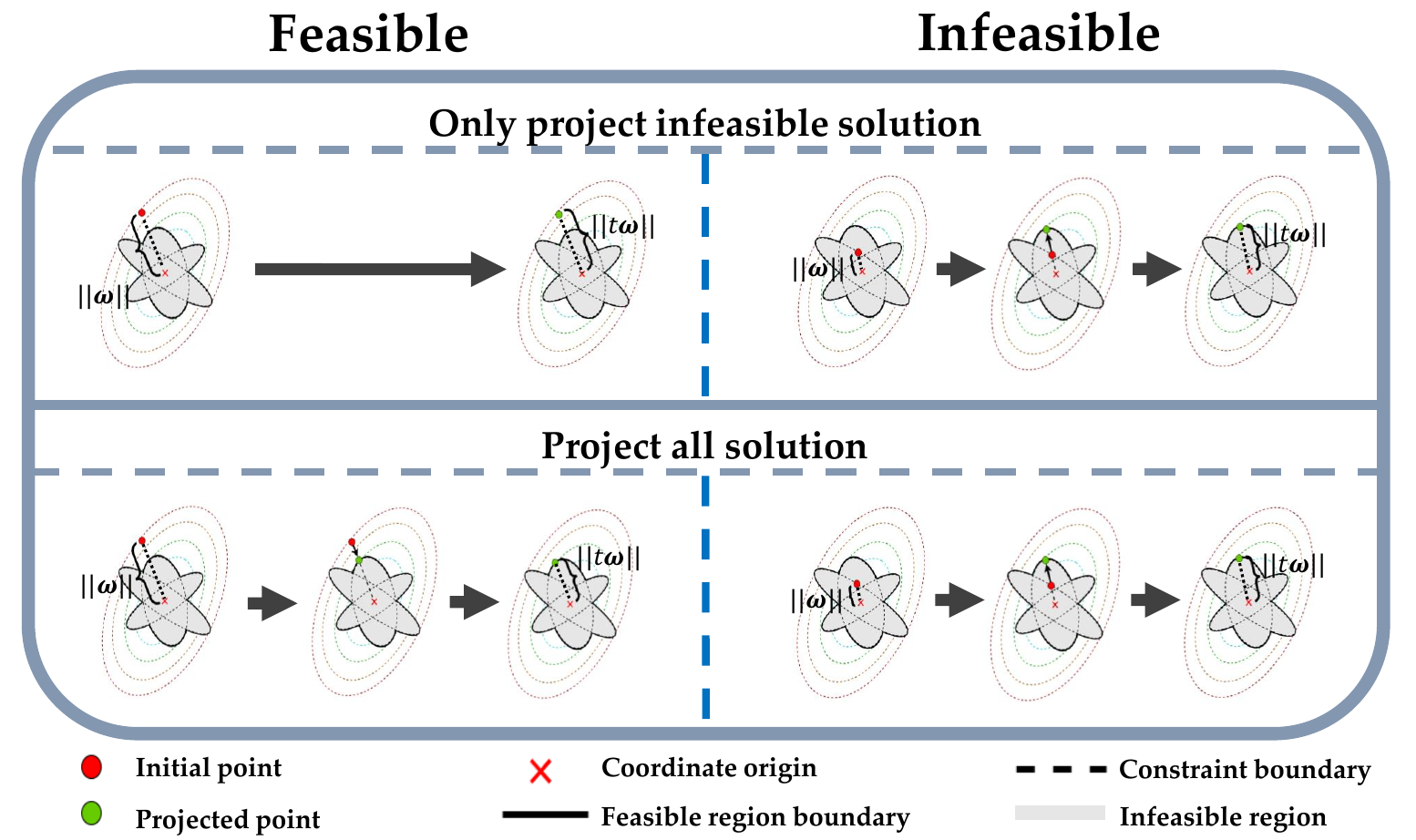}
    \caption{Illustration of projection.}
    \label{fig:projection}
\end{figure*}

\section{Reliable projection Enhanced Neural Network Method}
\subsection{Reliable Projection Function for Constrained MISO Beamforming}
As Theorem \ref{theorem-1} highlights, Problem \ref{p1} is NP-hard, suggesting that obtaining an optimal solution within polynomial time is unfeasible. Concurrently, the rapid advancements in NN technology provide a promising avenue for tackling NP-hard challenges. In the context of CO, employing NN-based strategies essentially involves training the network to approximate a function that maps the parameters of optimization and constraint equations directly to the optimal feasible solutions. The theoretical underpinning for this approach is supported by the universal approximation theorem, which asserts that a neural network with sufficient width can approximate any function to any desired degree of accuracy. However, practical constraints prevent the deployment of infinitely wide networks in real-world communication systems, leading inevitably to approximation errors as per the universal approximation theorem. Consequently, the output from a neural network may not always reside within the feasible region, nor can it guarantee inherently optimal solutions. Additionally, the opaque nature of neural networks complicates the analysis of conditions under which the network outputs points within the feasible region or the quality of the solutions it produces. To mitigate these issues, integrating a feature embedding module that compels the neural network to yield high-quality outputs within the feasible region becomes essential. Alternatively, employing a post-processing module to map potentially unreliable NN outputs back into the feasible region can ensure the performance quality of learning-to-optimize (L2O) schemes. Drawing inspiration from the successful application of softmax mapping in CO problems involving first-order linear sum constraints, a robust projection function has been devised, as depicted in Fig.~\ref{fig:system_inference}. This function is designed to facilitate a reliable CO solution leveraging neural network capabilities. The foundational design of this projection function is predicated on specific properties of Problems \ref{p1}, as initially delineated.

\begin{lemma}\label{lemma-1}
Given any non-zero vector $ \bm{w} $, there exists a scalar $ t $ such that the vector scaled by $ t $, expressed as $ t\bm{w} $, adheres to all constraints delineated in Problem \ref{p1}.
\end{lemma}

\begin{proof}
Consider the function 
\begin{align*}
    f_{i}(t) = \frac{\| \bm{h}_{i} t\bm{w} \|^2}{\sigma_{i}^2} = t^2 \frac{\| \bm{h}_{i} \bm{w} \|^2}{\sigma_{i}^2} = \alpha_i t^2, \quad \forall i\in \{1, \cdots, M\},
\end{align*}
where $ \alpha_i = \frac{\| \bm{h}_{i} \bm{w} \|^2}{\sigma_{i}^2} $ and by the premise $ \bm{w} \neq \bm{0} $, we have $ \alpha_i > 0 $ for all $ i $. As $ t $ approaches infinity $ t \rightarrow +\infty $, each $ f_{i}(t) $ unboundedly increases. Consequently, there must exist a sufficiently large $ t $ such that for all $ i $, the functions $ f_{i}(t) $ exceed the prescribed SNR thresholds $ \gamma_i $. 
\end{proof}

Leveraging Lemma \ref{lemma-1}, a preliminary approach could entail predefining a sufficiently large scalar $ \hat{t} $ and, upon determining that the neural network (NN) output vector $ \bm{w} $ fails to satisfy the constraints, simply scaling $ \bm{w} $ by $ \hat{t} $ to enforce compliance. While this technique guarantees constraint adherence for the NN's output, it suffers from two critical deficiencies. Initially, an excessive scale factor may inflate the magnitude of $ \|\hat{t}\bm{w}\|^2 $, detrimentally impacting the resolution of Problem \ref{p1}. Moreover, the determination of $ \hat{t} $ is not contingent on the NN's current input but rather on the statistical properties of $ \bm{h} $, rendering it unrelated to the NN's output. This approach is epitomized by a rudimentary step function:
\begin{align*}
    t=
    \begin{cases}
    1 & \text{if } \|\bm{h}_i^H \bm{w}\|^2/\sigma_i^2 \geq \gamma_i, \; \forall i \in \{1, \cdots, M\}, \\
    \hat{t} & \text{otherwise},
    \end{cases}
\end{align*}
In essence, for an NN predicated on backpropagation to refine its parameters and performance, discerning the association between the output, the scaled solution, and the optimization objective is non-trivial with such a step function. This is attributed to its non-differentiable nature or its zero-derivative characteristic, thus precluding the use of this elementary mechanism to simultaneously guarantee constraint satisfaction and optimal solution quality.

Inspired by Theorem \ref{theorem-2}, we propose a projection function, simple in construction yet robust in application as follows.
\begin{align}
    t =  \begin{cases}
    1   \qquad\qquad\qquad\quad \|\bm{h}_i^H \bm{w}\|^2/\sigma_i^2 \geq \gamma_i, \forall i \in \{1, \cdots, M\}, \\
    \max_{i \in \{1, \cdots, M\}} \sqrt{\frac{\gamma_{i}\sigma_{i}^{2}}{\|\bm{h}_{i}^{H}\bm{w}\|^2}} \qquad\qquad\qquad\qquad  \text{otherwise},
    \end{cases}
    \label{projection-1}
\end{align}

\begin{corollary}
Given any non-zero vectors $ \bm{h} $ and $ \bm{w} $, the projection function \eqref{projection-1} ensures that all constraints are satisfied upon scaling $ \bm{w} $ by $ t $ to yield $ t\bm{w} $.
\end{corollary}

\begin{proof}
Should any constraints be unmet, vector $ \bm{w} $ will be scaled by $ t $ to become $ \bm{x} = t\bm{w} $. Consequently, we can express:

\begin{align*}
    \|\bm{h}_{i}^{H}\bm{x}\|^2 &= \|\bm{h}_{i}^{H}t\bm{w}\|^2 = t^2\|\bm{h}_{i}^{H}\bm{w}\|^2,
\end{align*}
since $ t $ is defined as $ \max_{i\in\{1, \cdots, M\}}\sqrt{\frac{\gamma_{i}\sigma_{i}^{2}}{\|\bm{h}_{i}^{H}\bm{w}\|^2}} $, ensuring $ t^2 \geq \frac{\gamma_{i}\sigma_{i}^{2}}{\|\bm{h}_{i}^{H}\bm{w}\|^2} $ for all $ i $ in $ \{1, \cdots, M\} $. Therefore, the following function can be obtained.
\begin{align*}
    \|\bm{h}_{i}^{H}\bm{x}\|^2 &= t^2\|\bm{h}_{i}^{H}\bm{w}\|^2 \\
    &\geq \frac{\gamma_{i}\sigma_{i}^{2}}{\|\bm{h}_{i}^{H}\bm{w}\|^2}\|\bm{h}_{i}^{H}\bm{w}\|^2, \\
    &= \gamma_{i}\sigma_{i}^{2},
\end{align*}
thus all constraints can be satisfied.
\end{proof}

While the analysis validates that the projection function \eqref{projection-1} effectively project all infeasible points into the feasible region, an enhanced formulation, as described below, can optimize problem-solving performance:
\begin{align}
    t =  \max_{i \in \{1, \cdots, M\}} \sqrt{\frac{\gamma_{i}\sigma_{i}^{2}}{\|\bm{h}_{i}^{H}\bm{w}\|^2}}.
    \label{projection-2}
\end{align}
This modified projection function \eqref{projection-2} has the following properties.

\begin{corollary}\label{corollary-2}
Given any non-zero vectors $\bm{h}$ and $\bm{w}$, the projection function \eqref{projection-2} not only projects any infeasible solution into the feasible region but also enhances the performance of an initial solution that is already feasible.
\end{corollary}

\begin{proof}
Assuming $\bm{w}$ is initially feasible, it follows that for all $i$ in $\{1, \ldots, M\}$, the inequality $\|\bm{h}_i^H \bm{w}\|^2 \geq \gamma_i \sigma_i^2$ holds true. Consequently, each term under the square root in the definition of $t$ is less than or equal to 1, implying $t \leq 1$. Scaling $\bm{w}$ by $t$ as follows
\begin{align*}
    \|\bm{h}_k^H (t \cdot \bm{w})\|^2 &= t^2 \|\bm{h}_k^H \bm{w}\|^2 \\
    &= \left(\sqrt{\frac{\gamma_k \sigma_k^2}{\|\bm{h}_k^H \bm{w}\|^2}}\right)^2 \|\bm{h}_k^H \bm{w}\|^2 \\
    &= \gamma_k \sigma_k^2,
\end{align*}
where $k = \arg \max_{i \in \{1, \ldots, M\}} \sqrt{\frac{\gamma_i \sigma_i^2}{\|\bm{h}_i^H \bm{w}\|^2}}$. This calculation demonstrates that the scaled vector $t \cdot \bm{w}$ precisely meets the $k$-th constraint at equality, defining the boundary of the feasible region for that constraint.

Furthermore, for all $j$ in $\{1, \ldots, M\} \setminus \{k\}$, given that $\gamma_k \sigma_k^2 \geq \gamma_j \sigma_j^2$, the $j$-th constraint remains satisfied. This proof not only affirms that the projection function \eqref{projection-2} universally ensures feasibility but also indicates that using this function can effectively narrow the search space for solutions by focusing on the boundary conditions defined by the constraints.

Additionally, if $\bm{w}$ initially satisfies all constraints, the fact that $t \leq 1$ leads to:
\begin{align*}
    \|t \bm{w}\|^2 = t^2 \|\bm{w}\|^2 \leq \|\bm{w}\|^2,
\end{align*}
illustrating that the projection function \eqref{projection-2} can potentially enhance the performance of a solution that is already feasible.
\end{proof}

Then the Problem \ref{p1} can be reformulated as follows.
\begin{problem}\label{p2}
\begin{align}
& \min_{\bm{w} \in \mathbb{C}^N}
& & \| t\bm{w} \|^2 \\
& \text{s. t.}
& & t = \max_{i\in\{1,\cdots,M\}}\sqrt{\frac{\gamma_{i}\sigma_{i}^{2}}{\|\bm{h}_{i}^{H}\bm{w}\|^{2}}},
\end{align}
\end{problem}

\begin{theorem}\label{theorem-3}
The Problem \ref{p1} and Problem \ref{p2} are equivalent.
\end{theorem}
\begin{proof}
Let us consider the transformation $\bm{x} = t\bm{w}$, which allows us to reformulate the minimization problem as:
\begin{align*}
\min_{\bm{x}} \|\bm{x}\|^2 = \min_{\bm{w}} \|t\bm{w}\|^2.
\end{align*}
Here, $t$ is defined as $t = \max_{i \in \{1, \ldots, M\}} \sqrt{\frac{\gamma_i \sigma_i^2}{\|\bm{h}_i^H \bm{w}\|^2}}$, implying that:
\begin{align*}
t \geq \sqrt{\frac{\gamma_i \sigma_i^2}{\|\bm{h}_i^H \bm{w}\|^2}}, \quad \forall i \in \{1, \ldots, M\}.
\end{align*}
Squaring both sides of this inequality, we obtain:
\begin{align*}
t^2 \geq \frac{\gamma_i \sigma_i^2}{\|\bm{h}_i^H \bm{w}\|^2}, \quad \forall i \in \{1, \ldots, M\}.
\end{align*}
Consequently, the following inequality can be obtained.
\begin{align*}
\frac{\|\bm{h}_i^H \bm{x}\|^2}{\sigma_i^2} = \frac{\|\bm{h}_i^H t\bm{w}\|^2}{\sigma_i^2} = \frac{t^2 \|\bm{h}_i^H \bm{w}\|^2}{\sigma_i^2} \geq \gamma_i,
\end{align*}
demonstrating that scaling $\bm{w}$ by $t$ ensures all constraints of the form $\|\bm{h}_i^H \bm{w}\|^2/\sigma_i^2 \geq \gamma_i$ are satisfied, thereby confirming the feasibility of $\bm{x} = t\bm{w}$ in the context of the given optimization problem.
\end{proof}

\subsection{Projection-based L2O Method}
Although Problems \ref{p1} and \ref{p2} are equivalent, the transformation from Problem \ref{p1} to Problem \ref{p2} modifies the challenge from a constrained optimization problem concerning $\bm{w}$ into an unconstrained one. This adaptation enables the efficient and reliable application of NNs to solve the original problem. Additionally, it allows for a qualitative analysis of the differences in the complexity faced by NNs when tackling these problems. Firstly, projection function \eqref{projection-2} facilitates the NN's ability to project directly to the boundary of the feasible region, aligning along the vector from the origin, irrespective of the node output values. From a coordinate perspective, the NN needs only to determine a direction; projection function \eqref{projection-2} will then ascertain the optimal magnitude in that direction. Thus, when addressing Problems \ref{p1}, the NN is required to simultaneously identify both the angle and magnitude, whereas for Problems \ref{p2}, it needs only to discern the optimal angle. Invoking the universal approximation theorem, the neural complexity required to learn both the angle and magnitude in the solution space of Problems \ref{p1} surpasses that needed merely to learn the angle in Problems \ref{p2}. The latter scenario demands a simpler, lower-dimensional fitting function and thus fewer neurons for an adequate fit. Employing projection function \eqref{projection-2}, as proposed here, considerably reduces the structural complexity of the NN tasked with solving Problems \ref{p1}. Typically, the inferential complexity of an NN is directly associated with its structural complexity, which also diminishes the computational burden. Moreover, projection function \eqref{projection-2} significantly narrows the action space for the NN—from the entire feasible region to merely its boundary. According to \cite{xiao2022convergence}, the training convergence speed of an NN is linearly dependent on the size of the action space. Consequently, utilizing the proposed projection function not only expedites the inference process of the NN but also reduces the training overhead, enhancing overall computational efficiency.

Since MLP is the NN architecture with the simplest structure and the weakest characterization ability, and is most likely to be replaced by NNs with more complex structures such as CNN, GNN, and Transformer to improve performance. Therefore, in order to verify the performance and versatility of our proposed projection function, we use the simplest MLP with only one hidden layer as the NN architecture of this paper. Its specific architecture is as follows. Its input layer consists of $2NM$ neurons, which are used to extract the complex channel information $\bm{h}$ of $M$ users. The hidden layer has K neurons, and the output layer is $2N$ neurons, which are used to represent the beamforming vector with $N$ transmit antennas.

\begin{theorem}\label{theorem-4}
The inferencing complexity of the proposed projection-based neural network (NN) method for solving Problems \ref{p1} is $\mathcal{O}(NM)$.
\end{theorem}

\begin{proof}
As delineated by Goodfellow et al. \cite{goodfellow2016deep}, the computational complexity of the $i$-th layer in a neural network is characterized by:
\begin{align*}
\mathcal{M}_i = \zeta \phi_{\text{in}}\phi_{\text{out}} + \iota,
\end{align*}
where $\phi_{\text{in}}$ and $\phi_{\text{out}}$ represent the sizes of the input and output layers, respectively, while $\zeta$ and $\iota$ are coefficients reflecting the structural parameters of the NN. Given this, the inferencing complexity for the proposed NN architecture, incorporating $K$ intermediary layers between input layer $N$ and output layer $M$, is calculated as:
\begin{align*}
\zeta (NM \times K + K \times N) + 2\iota.
\end{align*}
Additionally, the computational effort required by the projection function is $\mathcal{O}(M)$, as it involves $M$ computations to ascertain the boundary of each constraint equation. Therefore, when combined, the overall inferencing complexity for addressing Problem \ref{p1} via the projection-based NN method is aggregated to $\mathcal{O}(NM)$.
\end{proof}
Theorem \ref{theorem-4} shows that the complexity of the projection-based NN method for solving Problem \ref{p1} is linear dependence with $N$ and $M$, respectively, which makes the method have good scalability.

\noindent\textbf{Remark 1}. We should highlights the distinct advantages of our proposed method over traditional Learning to Optimize (L2O) approaches that utilize penalty functions and constraint projection methods. Initially, employing a penalty function to train a neural network (NN) may reduce the likelihood of L2O solutions violating constraints. However, this approach cannot guarantee constraint compliance for all outputs. In contrast, our projection-based method ensures that every output from the NN strictly adheres to the constraints. Furthermore, the conventional constraint projection method achieves compliance by resolving a constrained optimization problem that seeks the closest point within the feasible region to the current point as the projection target. Yet, the complexity of solving this optimization problem often parallels that of addressing the original problem, as it involves the same constraints. This similarity can significantly compound the difficulty of finding a solution. Moreover, for points already within the feasible region, the constrained projection method does not modify their coordinates, hence offering no performance improvement for these points. Conversely, our analysis of Problem \ref{p1} reveals that the optimal solution invariably lies on the boundary of the feasible region. Thus, the projection method introduced in this work not only shifts all feasible region points to the boundary but also enhances the overall solution performance by exploiting this boundary condition.

\subsection{Unsupervised Training Method}
To minimize the training cost, we derive the gradient of the objective value of Problem 2 with respect to the parameters $\bm{\theta}$ of the neural network (NN), facilitating a label-free unsupervised training approach. First, we consider the gradient of the objective value with respect to the projected solution $t\bm{w}$, given by:
\begin{align}
\frac{\partial \|t\bm{w}\|^2}{\partial t\bm{w}} = 2t\bm{w}.
\end{align}
Since $t$ is a function of $\bm{w}$, we express:
\begin{align}
\frac{\partial t\bm{w}}{\partial \bm{w}} = t + \bm{w}\frac{\partial t}{\partial \bm{w}}.
\end{align}
Analyzing the gradient of $t$ with respect to $\bm{w}$, where $t$ is defined as $\max_{i \in \{1, \ldots, M\}} \sqrt{\frac{\gamma_i \sigma_i^2}{\|\bm{h}_i^H \bm{w}\|^2}}$, involves calculating the gradient of:
\begin{align}
\|\bm{h}_k^H \bm{w}\|^2 = (\bm{h}_k^H \bm{w})(\bm{h}_k^H \bm{w}) = \bm{w}^H \bm{H}_k \bm{w},
\end{align}
where $\bm{H}_k = \bm{h}_k \bm{h}_k^H$ is a Hermitian matrix formed by the outer product of $\bm{h}_k$ with itself, and $k = \arg \max_{i \in \{1, \ldots, M\}} \sqrt{\frac{\gamma_i \sigma_i^2}{\|\bm{h}_i^H \bm{w}\|^2}}$. Applying the chain rule and quotient rule, we derive the following equations.
\begin{align}
\frac{\partial\sqrt{\frac{\gamma_k \sigma_k^2}{\bm{w}^H \bm{H}_k \bm{w}}}}{\partial \bm{w}} = -\frac{1}{2} \left(\frac{\gamma_k \sigma_k^2}{(\bm{w}^H \bm{H}_k \bm{w})^{3/2}}\right) \frac{\partial \bm{w}^H \bm{H}_k \bm{w}}{\partial \bm{w}}.
\end{align}
For a quadratic form, the gradient is:
\begin{align}
\frac{\partial \bm{w}^H \bm{H}_k \bm{w}}{\partial \bm{w}} = 2 \bm{H}_k \bm{w}.
\end{align}
Substituting back, we obtain
\begin{align}
\frac{\partial t}{\partial \bm{w}}  = -\frac{\gamma_k \sigma_k^2 \bm{H}_k \bm{w}}{(\bm{w}^H \bm{H}_k \bm{w})^{3/2}}.
\end{align}
Thus, the derivative $\frac{\partial t\bm{w}}{\partial \bm{w}}$ is:
\begin{align}
&\frac{\partial t\bm{w}}{\partial \bm{w}} = -\frac{\gamma_k \sigma_k^2 \bm{H}_k \bm{w}}{(\bm{w}^H \bm{H}_k \bm{w})^{3/2}} + \max_{i \in \{1, \ldots, M\}} \sqrt{\frac{\gamma_i \sigma_i^2}{\|\bm{h}_i^H \bm{w}\|^2}},\\
&k = \arg \max_{i \in \{1, \ldots, M\}} \sqrt{\frac{\gamma_i \sigma_i^2}{\|\bm{h}_i^H \bm{w}\|^2}}
\end{align}
Accordingly, the parameter $\bm{\theta}$ can be updated using the gradient:
\begin{align}
\frac{\partial \|t\bm{w}\|^2}{\partial \bm{\theta}} = 2\frac{\partial t\bm{w}}{\partial \bm{w}}\bigg|_{\bm{w}=NN(h_1,\cdots,h_M;\bm{\theta})}\frac{\partial \bm{w}}{\partial \bm{\theta}}.\label{gradient}
\end{align}
Therefore, the NN can update the parameters by equation \eqref{gradient}.

\section{System Model Validation and Performance Evaluation with Simulation}
In this section, we evaluated the performance of our model in beamforming tasks through simulation tests, and compared it with traditional methods.

\subsection*{A. Individual Dataset Direct Execution}

Let's conduct unsupervised learning optimization simulations on a single dataset. The channels \( h_i \)'s are randomly and independently generated from the absolute value of a Gaussian distribution, expressed as \( |h_i| \sim \mathcal{N}(0, 1) \); the variance of the noise, $\sigma$, is set to one; the data were obtained from ten repetitive randomized experiments. The computations were performed on a system running Ubuntu 22.04, equipped with an Intel(R) Xeon(R) Silver 4214R CPU @ 2.40GHz and an NVIDIA RTX 3090 GPU. The theoretical lower bound is calculated via SDR, and feasible solutions are obtained by $10^5$ Gaussian randomizations. We have provided three models using the Reliable Projection Based Unsupervised Learning method. The NN model consists of three fully connected layers. Each layer is followed by a ReLU activation function during the forward pass.

We provide three models using the Reliable Projection Based Unsupervised Learning method.

\textbf{Model 1:} The model utilizes the projection method of function \ref{projection-1}, which only projects infeasible solutions, to optimize the neural network through unsupervised learning.

\textbf{Model 2:} The model utilizes the projection method of function \ref{projection-2}, which projects both feasible and feasible solutions, to optimize the neural network through unsupervised learning.

\textbf{Model 3:} The model first utilizes the penalty function method to pretrain the neural network, followed by employing the projection method from Equation \ref{projection-2} to project the solution into the feasible domain. During the pre-training phase, the loss function is defined as:
\begin{equation}
\text{Loss} = \left\| \mathbf{\omega} \right\|^2 - 0.2 \max \left\{0, \gamma_i - \frac{ \left\| h_i ^ H \mathbf{\omega} \right\|^2}{\sigma_i^2}\right\}
\end{equation}

All three models are trained using the Adam optimizer with an initial learning rate of 0.0005. Model 1 and Model 2 have a maximum of 1000 epochs for training, while Model 3 undergoes pre-training for 100 epochs followed by 900 epochs of training. 

We conducted performance testing of the model by varying $M$, $N$ and $\gamma_i$ individually.

\begin{enumerate}
    \item Varying $M$: We set $N$ as 50, and $\gamma_1 = \ldots = \gamma_M = \gamma = 10$ dB. We consider different values of $M$, namely 40, 80, 120, 160, and 200. The simulation results are depicted in Fig. \ref{fig:sqrt_user}.  The average loss for each model is depicted as bars, with accompanying black lines delineating the upper and lower performance bounds across different values of \( M \). From the figures, it is evident that both Model 2 and Model 3 significantly outperform Model 1 and SDR with randomization in terms of average losses and training time. Model 1, which only projects infeasible solutions, exhibits weaker convergence compared to Models 2 and 3. Moreover, as the parameter \( M \) increases, signifying a larger problem size, the Average Losses gradually increase. However, the impact on training time for projection-based methods is less pronounced than for SDR, suggesting a more scalable approach in terms of computational efficiency.
    
    \begin{figure}
        \centering
        \includegraphics[width=0.8\linewidth]{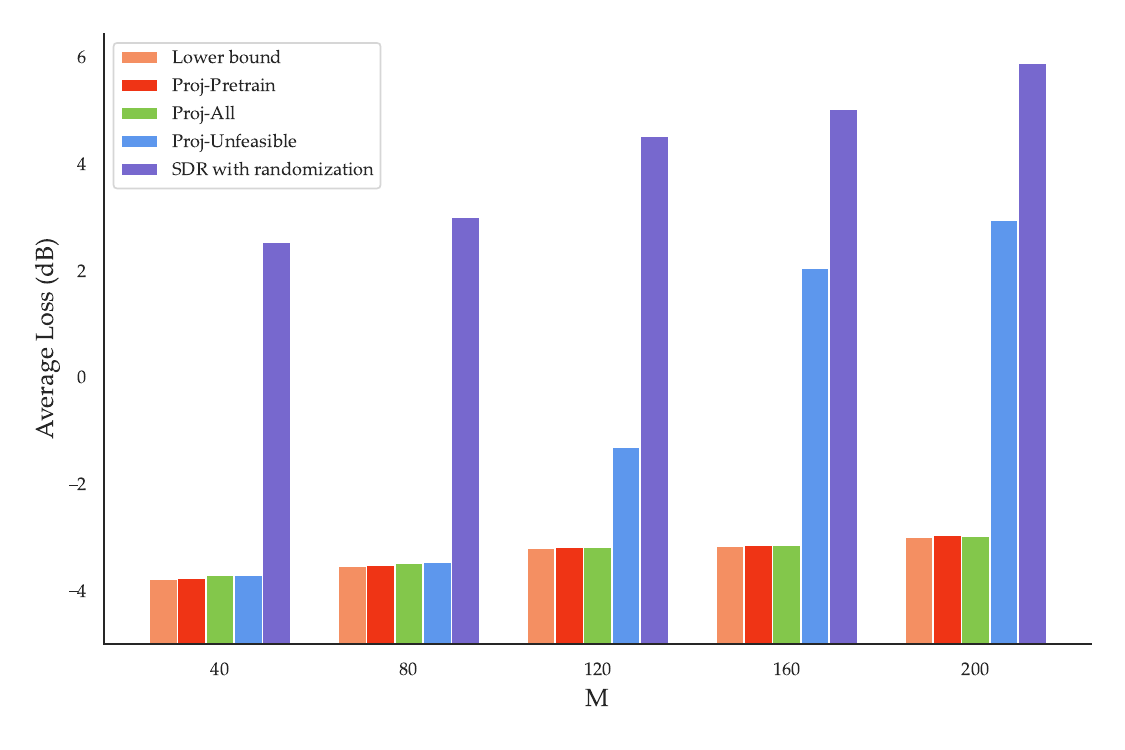}
        \caption{Average Losses with Changing $M$}
        \label{fig:sqrt_user}
    \end{figure}


    \begin{table}[h!]
    \centering
    \begin{tabular}{@{}lccccc@{}}
    \toprule
    M & 40 & 80 & 120 & 160 & 200 \\ \midrule
    SDR & 13.91 & 25.4102 & 38.5988 & 55.2094 & 63.4730 \\
    Proj-All & 1.6588 & 1.8064 & 1.7182 & 1.7892 & 1.8158 \\
    Proj-Unfeasible & 4.0894 & 4.8726 & 4.8836 & 4.8522 & 3.9776 \\
    Proj-Pretrain & 3.6954 & 3.9906 & 4.3228 & 3.7868 & 3.6310 \\ \bottomrule
    \end{tabular}
    \caption{Average Training Time with Changing M}
    \label{table:average_time_M}
    \end{table}
    
    \item Varying $N$: We keep $M = 80$ and $\gamma_1 = \ldots = \gamma_M = \gamma = 10$ dB, while varying $N$ with values of 10, 30, 50, 70, and 90 for experimentation. The results are depicted in Fig. \ref{fig:sqrt_N}. With increasing $N$, the problem size increases, but the difficulty of solving decreases. The average loss increases, the training time of the SDR increases, and the time of the projection-based methods decreases.
    
    \begin{figure}
        \centering
        \includegraphics[width=0.8\linewidth]{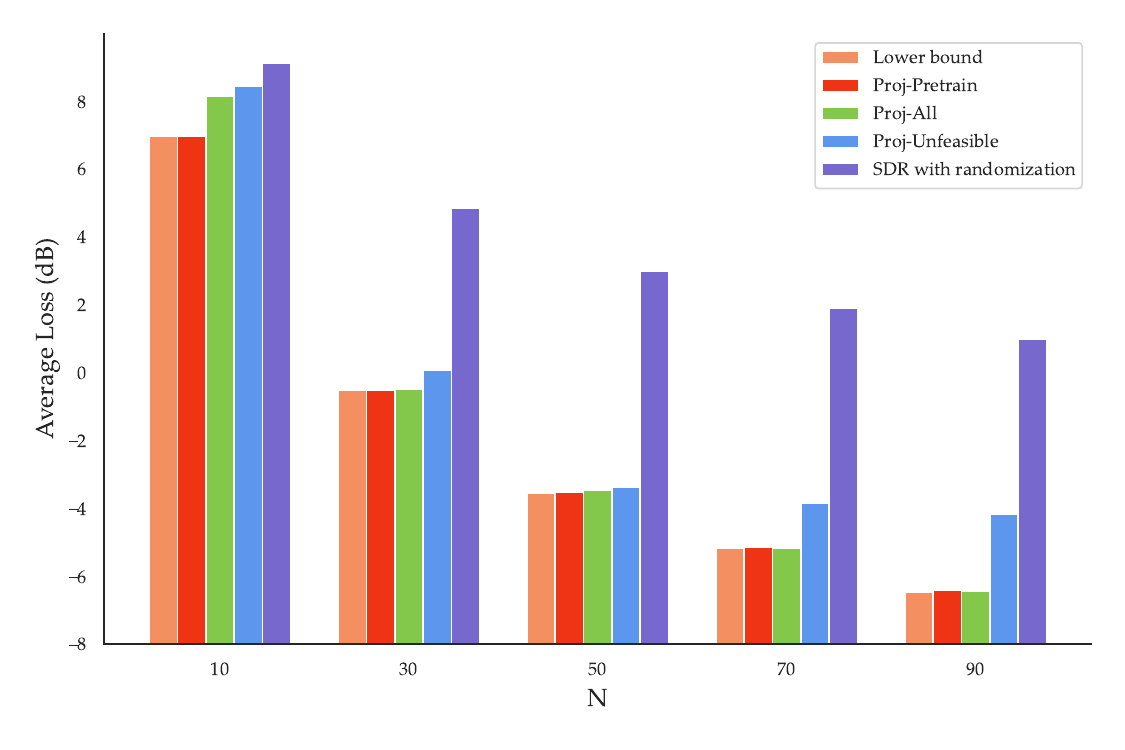}
        \caption{Average Losses with Changing $N$}
        \label{fig:sqrt_N}
    \end{figure}

    \begin{table}[h!]
    \centering
    \begin{tabular}{@{}lccccc@{}}
    \toprule
    N & 10 & 30 & 50 & 70 & 90 \\ \midrule
    SDR & 9.843 & 16.489 & 24.737 & 42.002 & 58.660 \\
    Proj-All & 7.92 & 4.13 & 1.869 & 1.835 & 1.693 \\
    Proj-Unfeasible & 29.909 & 23.936 & 15.023 & 15.006 & 9.314 \\
    Proj-Pretrain & 4.895 & 5.324 & 4.754 & 2.725 & 2.709 \\
    \bottomrule
    \end{tabular}
    \caption{Average Training Time with Changing N}
    \label{table:average_time_N}
    \end{table}
    
    \item Varying $\gamma_i$: We set $N = 50$ and $M = 80$, while varying $\gamma_1 = \ldots = \gamma_M = \gamma$ with values of 5dB, 10dB, 15dB, 20dB, and 25dB for experimentation. The results are shown in the figures. It is observed that with increasing $\gamma$, the problem complexity increases, leading to higher average losses. Due to its inherently difficult convergence nature, Model 2 often requires more time to converge under high $\gamma_i$.

    \begin{figure}
        \centering
        \includegraphics[width=0.8\linewidth]{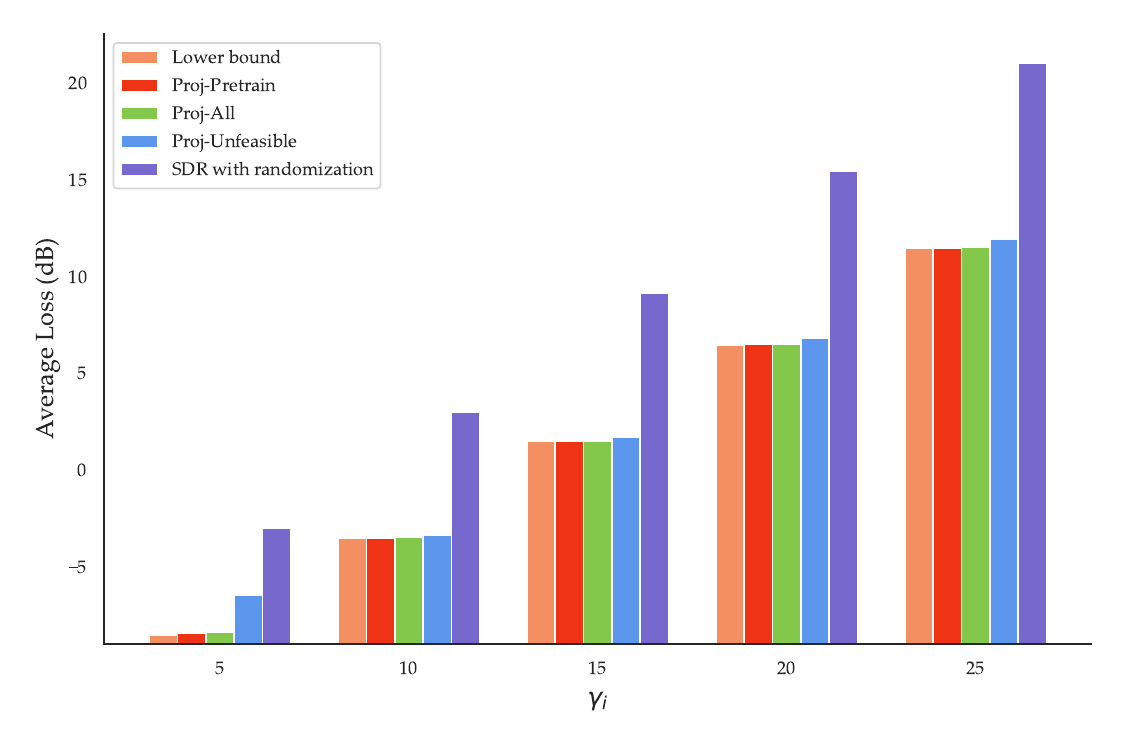}
        \caption{Average Losses with Changing $\gamma_i$}
        \label{fig:sqrt_gamma}
    \end{figure}
    
    \begin{table}[h!]
    \centering
    \begin{tabular}{@{}lccccc@{}}
    \toprule
    $\gamma_i$ & 5 & 10 & 15 & 20 & 25 \\ \midrule
    SDR & 30.223 & 27.633 & 28.758 & 27.889 & 28.495 \\
    Proj-All & 1.663 & 1.868 & 5.129 & 4.823 & 4.891 \\
    Proj-Unfeasible & 9.727 & 17.247 & 27.304 & 25.828 & 29.350 \\
    Proj-Pretrain & 2.282 & 5.277 & 5.125 & 5.050 & 5.116 \\ \bottomrule
    \end{tabular}
    \caption{Average Training Time with Changing $\gamma_i$}
    \label{table:average_time_gamma}
    \end{table}

\end{enumerate}

\subsection*{B. Training on Large-Scale Datasets}
We generated data for 10,000 instances of $h_i$ and split them into training and testing sets with an 8:2 ratio. Similarly, following the setup from previous simulations, all three models are trained using the Adam optimizer with an initial learning rate of 0.001, and the maximum training epochs are consistent with the previous experiments. The training process halts when the algorithm achieves $|w_{k+1} - w_k|^2_2 \leq 10^{-3}$ for 10 consecutive epochs. Additionally, during the training process, the test loss is recorded, and a batch size of 500 is employed.

\begin{enumerate}
    \item Varying $M$: Likewise, we retain the same configurations as previously described, with $N$ fixed at 50 and $\gamma_1 = \ldots = \gamma_M = \gamma = 10$ dB. We continue to examine various values of $M$, including 40, 80, 120, 160, and 200. From Fig. \ref{fig:batch_user}, it can be observed that the loss slightly increases with larger-scale datasets compared to single datasets, and the running time also increases. Meanwhile, the characteristic of Model 2 being difficult to converge is further magnified in larger-scale datasets.

    \begin{figure}
        \centering
        \includegraphics[width=0.8\linewidth]{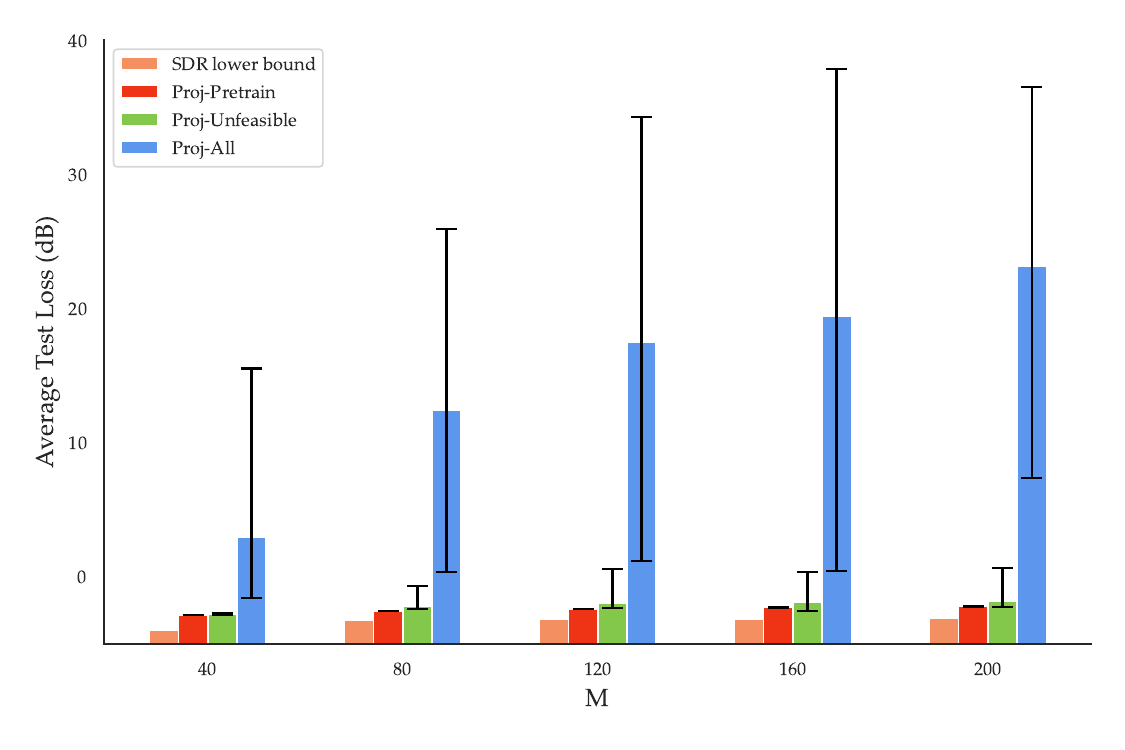}
        \caption{Average Test Losses with Changing $M$}
        \label{fig:batch_user}
    \end{figure}

    \item Varying $N$: We maintain the same parameter settings as in the previous experiments, with $M = 80$ and $\gamma_1 = \ldots = \gamma_M = \gamma = 10$ dB. The outcomes of these experiments are illustrated in Fig. \ref{fig:batch_N}.

    \begin{figure}
        \centering
        \includegraphics[width=0.8\linewidth]{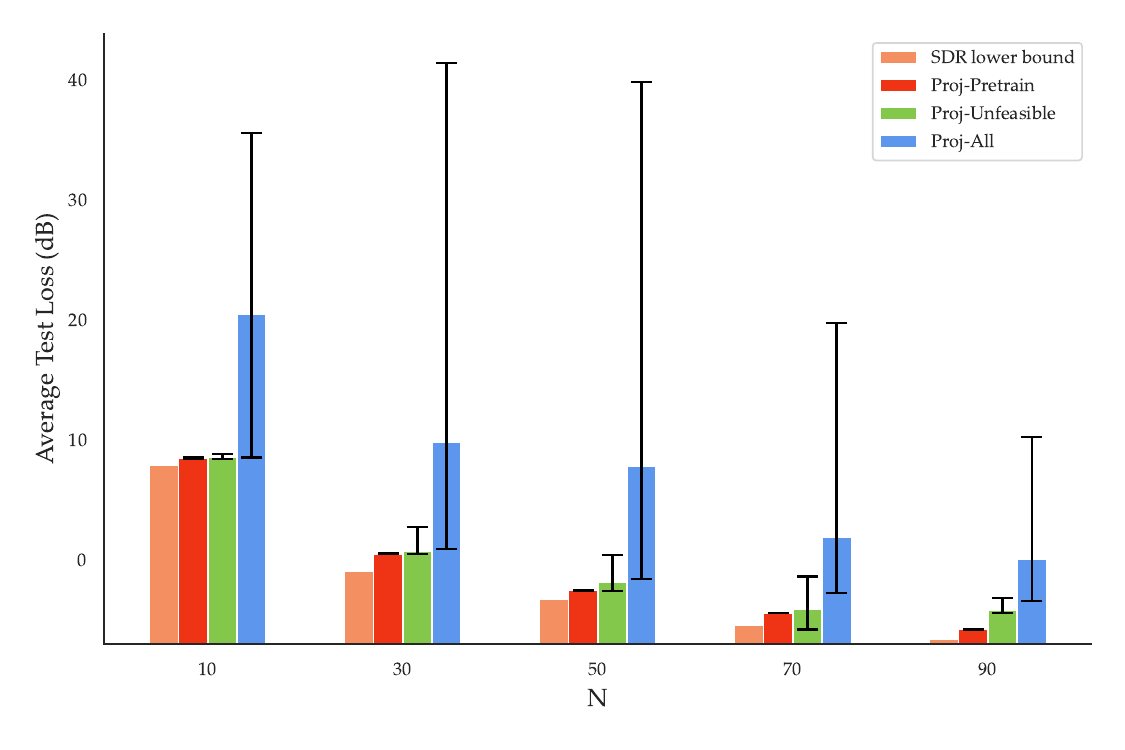}
        \caption{Average Test Losses with Changing $N$}
        \label{fig:batch_N}
    \end{figure}

    \item Varying $\gamma_i$: We establish $N = 50$ and $M = 80$, while exploring a range of $\gamma$ values from 5dB to 25dB in increments of 5dB for our experimentation. The results are illustrated in Fig. \ref{fig:batch_gamma}.

    \begin{figure}
        \centering
        \includegraphics[width=0.8\linewidth]{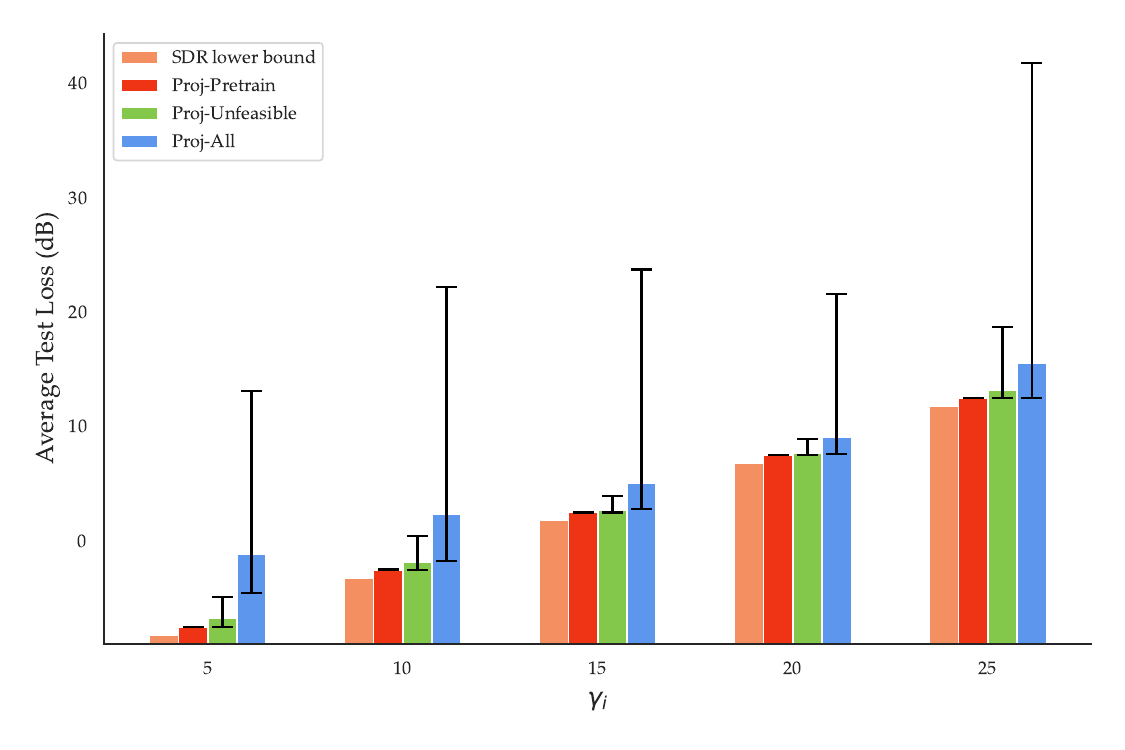}
        \caption{Average Test Losses with Changing $\gamma_i$}
        \label{fig:batch_gamma}
    \end{figure}

\end{enumerate}

\subsection*{C. Model Deployment Strategies}

\begin{figure}
    \centering
    \begin{minipage}[b]{0.45\textwidth}
        \centering
        \includegraphics[width=\textwidth]{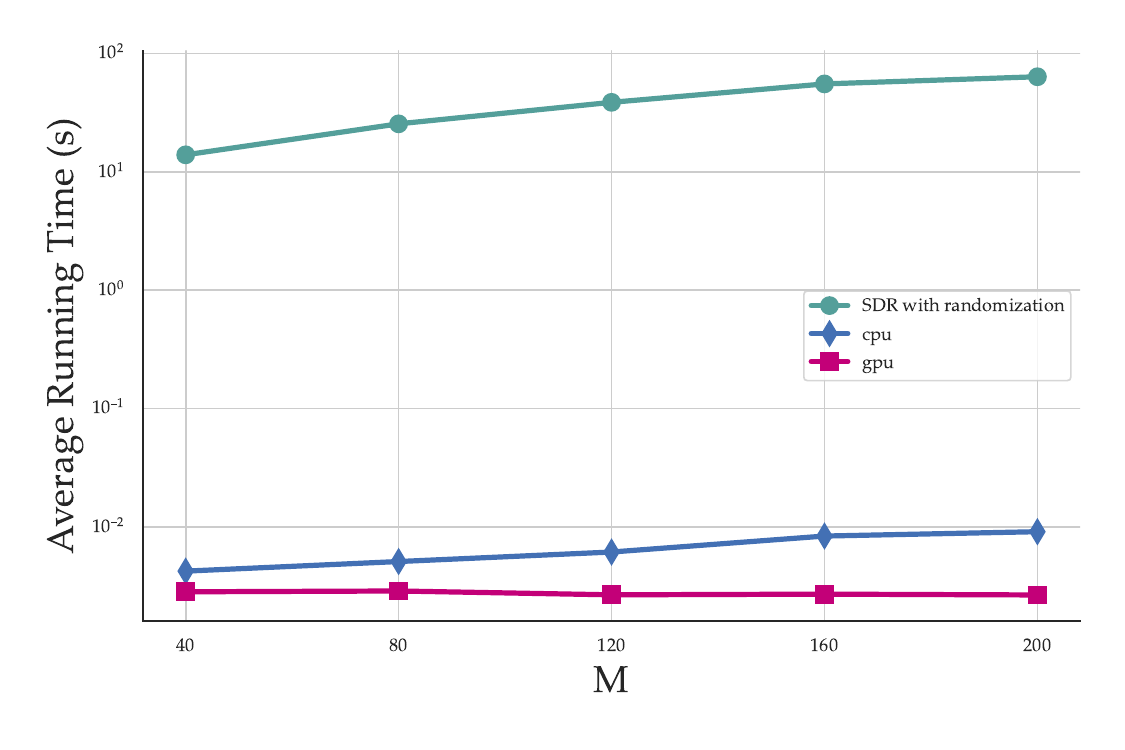}
        \label{fig:time_M.png}
    \end{minipage}
    \hspace{0.05\textwidth}
    \begin{minipage}[b]{0.45\textwidth}
        \centering
        \includegraphics[width=\textwidth]{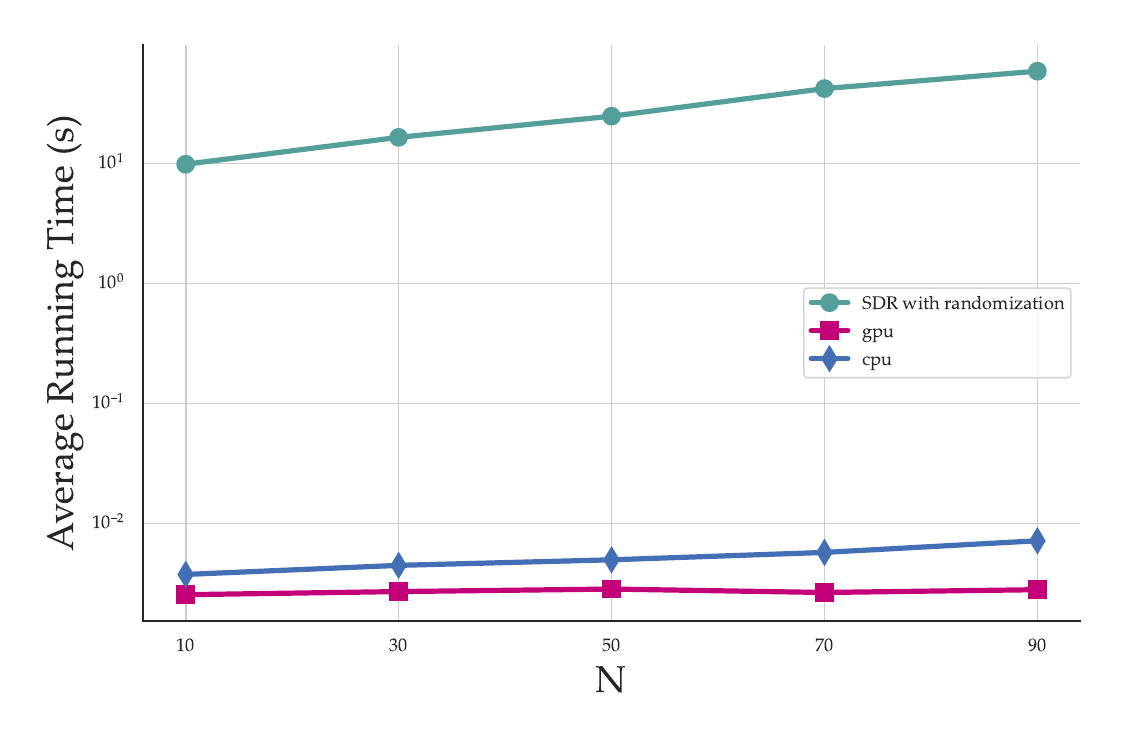}
        \label{fig:time_N.png}
    \end{minipage}
    \hspace{0.05\textwidth}
    \begin{minipage}[b]{0.45\textwidth}
        \centering
        \includegraphics[width=\textwidth]{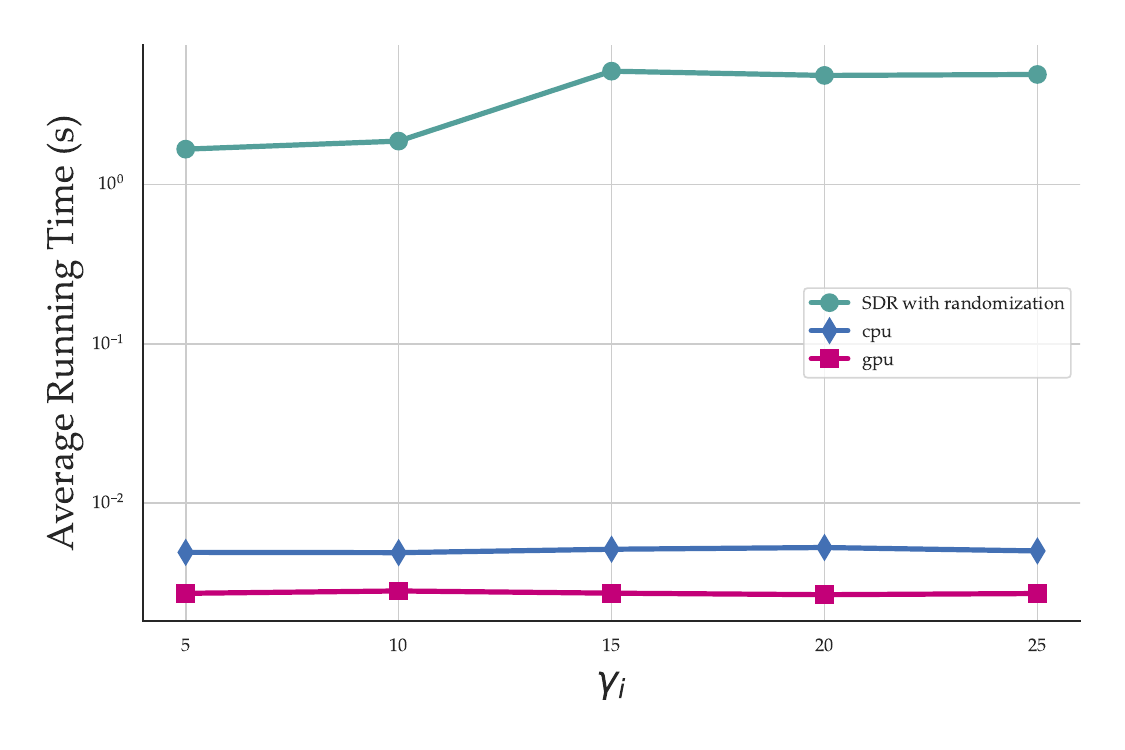}
        \label{fig:time_gamma.png}
    \end{minipage}
    \caption{Caption for the whole figure.}
    \label{fig:Comparison of Deployment Time}
\end{figure}

In this section, we examine the time cost associated with deploying the model in practical applications. We fine-tune pre-trained models using large-scale datasets for inference in real-world scenarios. Employing function \ref{projection-2}, we conduct inference on the new single dataset $h_i$ using the trained model, while varying the values of $M$, $N$, and $\gamma_i$. . From the figures, it is evident that projection-based learning optimization methods significantly outperform traditional SDR. They exhibit excellent timeliness and flexibility in practical applications.

\bibliography{ref}

\begin{thebibliography}{10}
\providecommand{\url}[1]{#1}
\csname url@samestyle\endcsname
\providecommand{\newblock}{\relax}
\providecommand{\bibinfo}[2]{#2}
\providecommand{\BIBentrySTDinterwordspacing}{\spaceskip=0pt\relax}
\providecommand{\BIBentryALTinterwordstretchfactor}{4}
\providecommand{\BIBentryALTinterwordspacing}{\spaceskip=\fontdimen2\font plus
\BIBentryALTinterwordstretchfactor\fontdimen3\font minus \fontdimen4\font\relax}
\providecommand{\BIBforeignlanguage}[2]{{%
\expandafter\ifx\csname l@#1\endcsname\relax
\typeout{** WARNING: IEEEtran.bst: No hyphenation pattern has been}%
\typeout{** loaded for the language `#1'. Using the pattern for}%
\typeout{** the default language instead.}%
\else
\language=\csname l@#1\endcsname
\fi
#2}}
\providecommand{\BIBdecl}{\relax}
\BIBdecl

\bibitem{6g}
N.~Cheng, F.~Chen, W.~Chen, Z.~Cheng, Q.~Yang, C.~Li, and X.~Shen, ``6g omni-scenario on-demand services provisioning: vision, technology and prospect(in chinese),'' \emph{Sci Sin Inform}, vol.~54, pp. 1025--1054,, 2024.

\bibitem{li2014special}
Q.~Li, Q.~Zhang, and J.~Qin, ``A special class of fractional qcqp and its applications on cognitive collaborative beamforming,'' \emph{IEEE Transactions on Signal Processing}, vol.~62, no.~8, pp. 2151--2164, 2014.

\bibitem{he2022qcqp}
X.~He and J.~Wang, ``Qcqp with extra constant modulus constraints: Theory and application to sinr constrained mmwave hybrid beamforming,'' \emph{IEEE Transactions on Signal Processing}, vol.~70, pp. 5237--5250, 2022.

\bibitem{wang2023scalable}
X.~Wang, N.~Cheng, L.~Fu, W.~Quan, R.~Sun, Y.~Hui, T.~Luan, and X.~S. Shen, ``Scalable resource management for dynamic mec: An unsupervised link-output graph neural network approach,'' in \emph{2023 IEEE 34th Annual International Symposium on Personal, Indoor and Mobile Radio Communications (PIMRC)}.\hskip 1em plus 0.5em minus 0.4em\relax IEEE, 2023, pp. 1--6.

\bibitem{wang2022joint}
X.~Wang, L.~Fu, N.~Cheng, R.~Sun, T.~Luan, W.~Quan, and K.~Aldubaikhy, ``Joint flying relay location and routing optimization for 6g uav--iot networks: A graph neural network-based approach,'' \emph{Remote Sensing}, vol.~14, no.~17, p. 4377, 2022.

\bibitem{goodfellow2016deep}
I.~Goodfellow, Y.~Bengio, and A.~Courville, \emph{Deep learning}.\hskip 1em plus 0.5em minus 0.4em\relax MIT press, 2016.

\bibitem{shen2020graph}
Y.~Shen, Y.~Shi, J.~Zhang, and K.~B. Letaief, ``Graph neural networks for scalable radio resource management: Architecture design and theoretical analysis,'' \emph{IEEE Journal on Selected Areas in Communications}, vol.~39, no.~1, pp. 101--115, 2020.

\bibitem{cao2023pcnet}
S.-Y. Cao, B.~Yu, L.~Luo, R.~Zhang, S.-J. Chen, C.~Li, and H.-L. Shen, ``Pcnet: A structure similarity enhancement method for multispectral and multimodal image registration,'' \emph{Information Fusion}, vol.~94, pp. 200--214, 2023.

\bibitem{wang2023cnn}
Z.~Wang, Z.~Yin, X.~Wang, N.~Cheng, Y.~Song, and T.~H. Luan, ``Cnn-based synergetic beamforming for symbiotic secure transmissions in integrated satellite-terrestrial network,'' in \emph{2023 IEEE 23rd International Conference on Communication Technology (ICCT)}.\hskip 1em plus 0.5em minus 0.4em\relax IEEE, 2023, pp. 1106--1111.

\bibitem{li2022radionet}
H.~Li, K.~Gupta, C.~Wang, N.~Ghose, and B.~Wang, ``Radionet: Robust deep-learning based radio fingerprinting,'' in \emph{2022 IEEE Conference on Communications and Network Security (CNS)}.\hskip 1em plus 0.5em minus 0.4em\relax IEEE, 2022, pp. 190--198.

\bibitem{gu2021knowledge}
Z.~Gu, C.~She, W.~Hardjawana, S.~Lumb, D.~McKechnie, T.~Essery, and B.~Vucetic, ``Knowledge-assisted deep reinforcement learning in 5g scheduler design: From theoretical framework to implementation,'' \emph{IEEE Journal on Selected Areas in Communications}, vol.~39, no.~7, pp. 2014--2028, 2021.

\bibitem{sidiropoulos2006transmit}
N.~D. Sidiropoulos, T.~N. Davidson, and Z.-Q. Luo, ``Transmit beamforming for physical-layer multicasting,'' \emph{IEEE transactions on signal processing}, vol.~54, no.~6, pp. 2239--2251, 2006.

\bibitem{boyd2004convex}
S.~P. Boyd and L.~Vandenberghe, \emph{Convex optimization}.\hskip 1em plus 0.5em minus 0.4em\relax Cambridge university press, 2004.

\bibitem{xiao2022convergence}
L.~Xiao, ``On the convergence rates of policy gradient methods,'' \emph{Journal of Machine Learning Research}, vol.~23, no. 282, pp. 1--36, 2022.

\end{thebibliography}
\bibliographystyle{IEEEtran}

\ifCLASSOPTIONcaptionsoff
  \newpage
\fi
\end{document}